\documentclass[letterpaper, 10pt, conference]{ieeeconf}

\usepackage{tikz}
\usepackage{pgfplots}
\usepackage{amsmath, amssymb, amsfonts, mathtools} 
\usepackage{bm}      
\usepackage{array, float, xcolor}                  
\usepackage{graphicx, epsfig}                      
\usepackage{tikz-cd}                               
\usepackage{wrapfig}                               
\usepackage{dsfont}                                
\usepackage[T1]{fontenc}                           
\usepackage{makecell}
\usepackage{epstopdf} 
\allowdisplaybreaks 

\usepackage{xcolor} 

\IEEEoverridecommandlockouts

\usepackage{amsthm}     

\let\labelindent\relax
\usepackage{enumitem}   

\newtheorem{theorem}{Theorem}
\newtheorem{lemma}[theorem]{Lemma}

\theoremstyle{definition}
\newtheorem{definition}{Definition}
\newtheorem{assumption}{Assumption}
\newtheorem{problem}{Problem}

\theoremstyle{remark}
\newtheorem{remark}{Remark}

\usepackage{cite} 

\makeatletter
\let\NAT@parse\undefined
\makeatother

\usepackage[colorlinks=true,
            linkcolor=blue,
            citecolor=blue,
            urlcolor=blue,
            pdfborder={0 0 0}]{hyperref}
\usepackage[nameinlink,capitalise]{cleveref}

\hypersetup{
    pdfauthor={Adeel Akhtar},
    pdftitle={GRaDS Lab},
    pdfdisplaydoctitle=true,
    pdfsubject={Unicycle Robots, QP based control},
    pdfkeywords={nonlinear systems, dynamic feedback linearization},
    pdfcreator={LaTeX with hyperref package},
    pdfproducer={dvips + ps2pdf},
    bookmarksnumbered,
    pdfstartview={FitV}
}

\usepackage{macros} 

\definecolor{RubineRed}{RGB}{209,0,81}

\definecolor{olivegreen}{rgb}{0.14,0.29,0}

\def\BibTeX{{\rm B\kern-.05em{\sc i\kern-.025em b}\kern-.08em
    T\kern-.1667em\lower.7ex\hbox{E}\kern-.125emX}}

\pgfplotsset{compat=1.18}

\begin{document}

\title{A Relaxed Quadratic-Program-based Framework for Trajectory Tracking of Unicycle Robots with Singularity Avoidance}

\makeatletter
\@ifclassloaded{ieeeconf}{
  \author{Hamza Tariq$^{1}$, Usman Ali$^{2}$, and Adeel Akhtar$^{1}$
\thanks{$^{1}$Hamza Tariq and Adeel Akhtar are with the Department of Mechanical \& Industrial Engineering at the New Jersey Institute of Technology, Newark, NJ 07102, USA.
        {\tt\small \{ht299, adeel.akhtar\}@njit.edu}}%
\thanks{$^{2}$Usman Ali is with the School of Engineering and Sustainable Development at the De Montfort University, Leicester, LE1 9BH, UK.
        {\tt\small usman.ali@dmu.ac.uk}}%
}
  
}{
  \author{%
    \IEEEauthorblockN{Adeel Akhtar}
    \IEEEauthorblockA{Department of Mechanical and Industrial Engineering \\ at the New Jersey Institute of Technology (NJIT), NJ, USA \\ 
    {\tt\small adeel.akhtar@njit.edu}}
    \and
    \IEEEauthorblockN{Usman Ali}
    \IEEEauthorblockA{School of Engineering, Infrastructure and Sustainability\\ De Montfort University, Leicester, LE1 9BH, UK \\
    {\tt\small usman.ali@dmu.ac.uk}}
  }
}

\makeatother

\maketitle

\begin{abstract}
Dynamic feedback linearization (DFL) is a classical technique for trajectory tracking of unicycle-type mobile robots, but  the resulting DFL-based controller becomes singular when the linear velocity vanishes, rendering standard DFL-based controllers unsuitable for stop-and-reverse maneuvers. This paper proposes a quadratic-program (QP)-based optimal control framework that avoids this singularity, while establishing local Lipschitz continuity of the resulting feedback law. Our approach reformulates the DFL constraints as an equality-constrained QP with a slack variable, ensuring feasibility for all states and reference signals, including at points where the robot's velocity vanishes. By introducing slack variables and tunable parameters, we demonstrate that the singular configuration can be avoided for a large class of reference trajectories. The effectiveness of the proposed approach for trajectory tracking is demonstrated through ROS2–Gazebo simulations on a TurtleBot3 Waffle robot. The code is available at \url{https://gradslab.github.io/DFL_QP_Unicycle/}.
\end{abstract}

\section{INTRODUCTION}

Nonholonomic robots, and in particular unicycle-type robots, are ubiquitous in robotic applications but are challenging to control due to their underactuated dynamics and nonholonomic constraints~\cite{Padensurvey2016}. While trajectory tracking for such systems is extensively studied in the literature, the synthesis of time-invariant, locally Lipschitz controllers remains challenging due to the presence of singular configurations, where the controller becomes undefined~\cite{AkhWasNie2013Journal,hirschorn1987output}.

Dynamic feedback linearization (DFL) is one of the most commonly used approaches for trajectory tracking of unicycle-type robots~\cite{d1995control,de1991exponential}. By dynamically extending the system and rendering the input–output map linear, DFL enables the application of linear control techniques for tracking~\cite{isidori1985nonlinear}. However, it is known that the DFL controller for unicycle robots exhibits a singularity at zero translational velocity, at which the feedback law becomes undefined. As a result, classical DFL-based controllers are not directly applicable during stop-and-reverse maneuvers~\cite{oriolo2002wmr, WanAkhSan2025}.

Many solutions to this singularity issue have been proposed in the literature. Some approaches rely on time-varying or oscillatory feedback laws, which can overcome topological obstructions inherent to nonholonomic systems~\cite{samson2002control}. Other strategies restrict operation to nonsingular regimes or employ switching and hybrid control architectures~\cite{klanvcar2007tracking,astolfi1996discontinuous,tomlin1998switching}. Although effective in practice, these approaches may sacrifice time-invariance or continuity. Time-invariant controllers are often preferred because they simplify implementation, stability analysis, and robustness to timing uncertainties.

In recent years, optimization-based feedback controllers, such as control Lyapunov and control barrier function-based quadratic programs (QPs), have gained increasing attention due to their flexibility and ability to encode constraints~\cite{ames2019control}. However, the solution mappings of parametric quadratic programs are, in general, only piecewise smooth and may lose Lipschitz continuity at points of degeneracy~\cite{bonnans2013perturbation}. As a consequence, Lipschitz continuity of an optimization-based feedback law must be explicitly established.

Motivated by these observations, this paper addresses the following question: \textit{Can a DFL–based trajectory tracking controller for unicycle-type robots be reformulated to circumvent the issue of singular configurations while preserving time-invariance and establishing Lipschitz continuity of the feedback law?}
In our work, the objective is not asymptotic stabilization of a fixed equilibrium point, which is precluded by Brockett's condition~\cite{brockett1983asymptotic}, but rather Cartesian trajectory tracking through zero velocity and direction reversal. Specifically, the heading is left free, allowing the controller to track a planned path using either forward or backward motion.

In this paper, we propose a QP-based framework that relaxes the classical DFL constraints via a slack variable. This relaxation ensures feasibility for all states, including configurations at which the DFL decoupling matrix loses rank. We show that the resulting optimal controller is locally Lipschitz continuous. While existing tracking results impose persistence-of-excitation conditions on reference trajectories~\cite{wang2015simultaneous}, our proposed framework does not require persistent motion of the reference. Our relaxed QP framework avoids the singularity at the expense of exact trajectory tracking at time instances when the classical DFL constraints become singular, using slack variables to keep the QP feasible. 

To the best of our knowledge, this is the first time that the singularity resolution of the classical DFL has been proposed using a relaxation-based QP framework, which enables direction reversal along the desired trajectories.  
The effectiveness of the proposed approach is illustrated through ROS2–Gazebo numerical experiments on a TurtleBot3 robot, and a comparison with the classical DFL controller. 
The main contributions of the paper are summarized as follows:
\begin{enumerate}
    \item a relaxed quadratic-programming framework based on dynamic feedback linearization for synthesizing a trajectory tracking controller for unicycle-type robots, which remains well-defined at zero velocity;
    \item a proof of local Lipschitz continuity of the proposed controller (Theorem~\ref{th:lip_con_closed_form});
    \item a proof of uniform ultimate boundedness of the tracking error under the stated assumptions (Theorem~\ref{thm:uub_after_exit}).
\end{enumerate}

\section{Notation and Background}\label{sec:2}
\subsection{Notation}
For a vector $\mathbf{x}\in\mathbb{R}^n$, we write
$\mathbf{x}=[x_1,\dots,x_n]^\top$, where $(\cdot)^\top$ denotes transpose. A twice continuously differentiable reference signal is denoted by
$\mathbf{y}_{\mathrm{ref}}\in C^2([0,\infty);\mathbb{R}^n)$. The identity matrix in $\mathbb{R}^{n\times n}$ is denoted by $I_n$. The Euclidean norm of a vector $\mathbf{v}\in\mathbb{R}^n$ is denoted by
$\|\mathbf{v}\|$, and for a matrix $E\in\mathbb{R}^{n\times m}$, $\|E\|$ denotes its induced $2$-norm. The notation $E\succ0$ means that $E$ is positive definite. For matrices $F\in\mathbb{R}^{n\times n}$ and $G\in\mathbb{R}^{m\times m}$, $\operatorname{blkdiag}(F,G)$ denotes the corresponding block-diagonal matrix. The determinant and rank of $E$ are denoted by $\det(E)$ and $\operatorname{rank}(E)$, respectively.

\subsection{Preliminaries}\label{sec:2b}
The dynamics of a control-affine system are defined by a drift vector field $f: \mathbb{R}^n \to \mathbb{R}^n$, and a matrix of control vector fields $g: \mathbb{R}^n \to \mathbb{R}^{n \times m}$. For each state $\mathbf{x} \in \mathbb{R}^n$ and input $\mathbf{u} \in \mathbb{R}^m$, the corresponding state derivative is given by
\begin{equation}\label{eq:control_affine_sys}
\dot{\mathbf{x}} = f(\mathbf{x}) + g(\mathbf{x})\mathbf{u}.
\end{equation}
A static state-feedback law is given by a mapping $\kappa: \mathbb{R}^n \to \mathbb{R}^m$ which assigns each state $\mathbf{x} \in \mathbb{R}^n$ to a corresponding control input $\mathbf{u} = \kappa(\mathbf{x})$. Substituting this feedback law into the system dynamics~\eqref{eq:control_affine_sys} yields the closed-loop system, 
\begin{equation} \label{eq:closedloop}
\dot{\mathbf{x}} = f(\mathbf{x}) + g(\mathbf{x})\kappa(\mathbf{x}).
\end{equation}

\begin{definition}[Local Lipschitz Continuity~\cite{khalil2002nonlinear}]
A map $\kappa:\mathcal{D}\to\mathbb{R}^m$ is said to be \emph{locally Lipschitz} on a set
$\mathcal{D}\subset\mathbb{R}^n$ if for every compact set $K\subset\mathcal{D}$ there exists a constant
$L_K>0$ such that
 $\|\kappa(x_1) - \kappa(x_2)\| \leq L_K \|x_1 - x_2\|,$
for all $x_1, x_2 \in K$. If $f$ and $g$ are locally Lipschitz and $\kappa$ is locally Lipschitz, then the closed-loop vector field
$x \mapsto f(x)+g(x)\kappa(x)$ is locally Lipschitz, guaranteeing existence and uniqueness of solutions
of~\eqref{eq:closedloop}~\cite[Thm.~3.1]{khalil2002nonlinear}.
\end{definition}
\noindent
\begin{definition}[Uniformly ultimately bounded (UUB)~\cite{khalil2002nonlinear}]
\label{def:UUB}
The system $\dot x=f(t,x)$ is said to be \emph{UUB} if there exists a constant $b>0$
such that for every $a>0$, there exists a time $T=T(a,b)\ge 0$, independent of $t_0\ge 0$, for which
$\|x(t_0)\|\le a \;\implies \; \|x(t)\|\le b,\text{ for all }\, t\ge t_0+T.$
\end{definition}

\subsection{Background of Dynamic Feedback Linearization~(DFL) for the Unicycle Robot}

Let $q = (x_1, x_2, x_3) \in \mathbb{R}^2\times \mathbb{S}^1$ denote the configuration of a unicycle-modeled robot, where $(x_1,x_2)\in\mathbb{R}^2$ is the planar position and $x_3 \in \mathbb S^1$ is the heading angle. The control input is $\nu = (v,\omega) \in \mathbb{R}^2$, where $v$ is the linear velocity and $\omega$ is the angular velocity. The unicycle kinematics are
\begin{equation}\label{eq:unicycle_model}
    \dot{q} = g_v(q)v + g_\omega(q)\omega,
\end{equation}
where
$
g_v(q) = 
\begin{bmatrix} 
\cos x_3 &
\sin x_3 & 
0 
\end{bmatrix}^\top
$ and $
g_\omega(q) = 
\begin{bmatrix} 
0 &
0 & 
1 
\end{bmatrix}^\top.
$ The output of interest is the Cartesian position
\begin{equation}\label{eq:output}
\mathbf{y} = (y_1,y_2) \eqdef h(\mathbf{x}) = (x_1,x_2).
\end{equation}
It is well established that this output is not input-output linearizable through static state feedback for the standard unicycle model~\cite{de1991exponential,oriolo2002wmr}. Therefore, following the dynamic feedback linearization approach, we extend the system by treating the linear velocity $v$ as a state variable $x_4 \eqdef v$ and introducing the acceleration $a \eqdef \dot v$ as a new control input. The extended state is
$\mathbf{x} = (x_1,x_2,x_3,x_4)\in\mc{X}\subset\mathbb R^4,$ and the new input is
$\mathbf u \eqdef (u_1,u_2) = (\omega,a)\in\mathbb R^2.$
The extended dynamics can be written in control-affine form as
\begin{equation}\label{eq:extended_model}
    \dot{\mathbf{x}} = F(\mathbf{x}) + \sum_{i=1}^{2} g_i(\mathbf{x})u_i,
\end{equation}
where $F:\mathcal{X}\to\mathbb{R}^4$ is the drift vector field and 
$g_i:\mathcal{X}\to\mathbb{R}^4$, $i\in\{1,2\}$, are the control vector fields given by
$
F(\mathbf{x}) =
\begin{bmatrix}
x_4 \cos x_3 &
x_4 \sin x_3 &
0 &
0
\end{bmatrix}^\top,\
g_1(\mathbf{x}) =
\begin{bmatrix}
0 &
0 &
1 &
0
\end{bmatrix}^\top,
\text{ and }
g_2(\mathbf{x}) =
\begin{bmatrix}
0 &
0 &
0 &
1
\end{bmatrix}^\top.
$
We can now differentiate the output~\eqref{eq:output} twice to get
\begin{equation}
\ddot{\mathbf{y}} = D(\mathbf{x})\mathbf{u},
\label{eq:linear_form}
\end{equation}
where the decoupling matrix is the map $D:\mathcal{X} \to \mathbb{R}^{2 \times 2}$ that assigns to each state $\mathbf{x} \in \mathcal{X}$ the matrix
\begin{equation}
D(\mathbf{x}) =
\begin{bmatrix}
-x_4 \sin x_3 & \cos x_3 \\
x_4 \cos x_3 & \sin x_3
\end{bmatrix}.
\label{eq:decoupling_matrix}
\end{equation}
The decoupling matrix $D(\mathbf{x})$ is invertible whenever its determinant, $\det(D(\mathbf{x})) = -x_4$, is nonzero. For any $x_4 \neq 0$, we define a new virtual input $\boldsymbol{\eta}$ and a control law $\mathbf{u} = D(\mathbf{x})^{-1} \boldsymbol{\eta}$,
which transforms the nonlinear dynamics into a double-integrator system 
\begin{equation}
\ddot{\mathbf{y}} = \boldsymbol{\eta}. 
\label{eq:double_integrator}
\end{equation}
Now, let $\mathbf y_{\mathrm{ref}}\in C^2([0,\infty);\mathbb R^2)$ and define the reference signal $r(t) \in\mathbb R^6$ as
\begin{equation}\label{eq:ref}
    r(t)\eqdef\big(\mathbf y_{\mathrm{ref}}(t),\dot{\mathbf y}_{\mathrm{ref}}(t),\ddot{\mathbf y}_{\mathrm{ref}}(t)\big).
\end{equation}
Similar to~\cite{oriolo2002wmr}, we can design a controller for the extended system~\eqref{eq:extended_model} to track a twice-differentiable reference $\mathbf y_{\mathrm{ref}}$. This is done by defining the tracking error $\mathbf{e} \eqdef \mathbf{y} - \mathbf{y}_{\text{ref}}$ and applying a proportional-derivative (PD) controller
\begin{equation}
\boldsymbol{\eta}(\mathbf{x},r) = \ddot{\mathbf{y}}_{\text{ref}} - K_d \dot{\mathbf{e}} - K_p \mathbf{e},
\label{eq:virtual_input}
\end{equation}
where $K_p$ and $K_d$ are positive-definite diagonal gain matrices. This choice yields exponentially stable error dynamics whenever $D(\mathbf{x})$ is invertible. However, since $\det(D(\mathbf{x}))=-x_4$, the controller becomes undefined at $x_4=0$, i.e., during zero-velocity instants such as stop-and-reverse maneuvers. For the extended state $\mathbf{x} \in \mathcal{X},$ and the control input
$\mathbf{u} \in \mathbb{R}^2,$ plugging the PD control law~\eqref{eq:virtual_input} into~\eqref{eq:double_integrator} and~\eqref{eq:linear_form} gives  
\begin{equation}
D(\mathbf{x}) \mathbf{u} = \boldsymbol{\eta}(\mathbf{x},r),
\label{eq:dfl_constraint}
\end{equation}
where $\boldsymbol{\eta}(\mathbf{x},r)$ is given by~\eqref{eq:virtual_input}.  
Note that any $\mathbf{u}$ that satisfies~\eqref{eq:dfl_constraint} achieves \emph{strict} Dynamic Feedback Linearization (DFL) of the extended unicycle system~\eqref{eq:extended_model}.

\section{Problem Formulation}\label{sec:3}
We next formulate a relaxed QP that trades asymptotic tracking with singularity resolution at zero translational velocity. Let $Q=\operatorname{diag}(q_\omega,q_a)\succ 0$ be a positive definite weight matrix on the control input. A ``naively'' formulated min-norm Quadratic Program (QP) that satisfies this DFL condition~\eqref{eq:dfl_constraint} is given as follows:
\begin{equation}
\begin{aligned}
\min_{\mathbf{u}} \quad & \mathbf{u}^\top Q \mathbf{u}, \\
\text{s.t.} \quad & D(\mathbf{x})\mathbf{u} = \boldsymbol{\eta}(\mathbf{x},r).
\end{aligned}
\label{eq:fbl_qp}
\end{equation}
However, the QP-based controller defined by~\eqref{eq:fbl_qp} still faces the limitation of classical DFL for the unicycle robot, i.e., the singularity at zero velocity. In this state, i.e., $x_4 = 0$, the decoupling matrix $D(\mathbf{x})$ loses rank and the constraint~\eqref{eq:dfl_constraint} is solvable only if the desired virtual input $\boldsymbol{\eta}$ lies in the span of the vector $(\cos x_3,\, \sin x_3)$; otherwise, no control input $\mathbf{u}$ can satisfy~\eqref{eq:dfl_constraint}, and the QP~\eqref{eq:fbl_qp} becomes infeasible. 

To remain feasible at zero velocity, we relax the DFL equality with a slack variable $\bm{\delta}\in\mathbb R^2$ and add an acceleration bias to mitigate deadlock at rest. To define the bias, we need some auxiliary functions. First, we define the unit heading and its normal direction as 
\begin{equation}\label{eq:unit_vecs}
    \mathbf{n}_\parallel(\mathbf x) \eqdef [\cos x_3,\ \sin x_3]^\top, \;
    \mathbf{n}_\perp(\mathbf x) \eqdef [-\sin x_3,\ \cos x_3]^\top,
\end{equation}
respectively. Next, we define the scalar projections of $\bm{\eta}(\mathbf x,r)$ onto these directions as
\begin{equation}\label{eq:scalar_proj}
s_\parallel(\mathbf x,r) \eqdef \mathbf{n}_\parallel(\mathbf x)^\top \bm{\eta}(\mathbf x,r),\;\;
{s}_\perp(\mathbf x,r) \eqdef \mathbf{n}_\perp(\mathbf x)^\top \bm{\eta}(\mathbf x,r).
\end{equation}
Now, with these projections and a fixed scalar $l > 0$, we define bounded weight functions as
\begin{equation}\label{eq:rhos}
    \rho_{\parallel}(\mathbf x,r) \eqdef
\frac{s_{\parallel}(\mathbf x,r)^2}{s_{\parallel}(\mathbf x,r)^2 + l^2},\quad
\rho_{\perp}(\mathbf x,r) \eqdef
\frac{s_{\perp}(\mathbf x,r)^2}{s_{\perp}(\mathbf x,r)^2 + l^2}.
\end{equation}
With these definitions, our final proposed relaxed QP is
\begin{subequations}
\label{eq:fbl_qp_relaxed}
\begin{align}
\min_{\mathbf{u},\,\bm{\delta} \in \mathbb{R}^2} \quad & \frac{1}{2}\mathbf{u}^\top Q \mathbf{u} + \frac{1}{2}\bm{\delta}^\top P \bm{\delta} + \bm{c}(\mathbf{x},r)^\top \mathbf{u}, \label{eq:fbl_qp_relaxed_obj} \\
\text{s.t.} \quad & D(\mathbf{x}) \mathbf{u} = \boldsymbol{\eta}(\mathbf{x},r) + \bm{\delta}, \label{eq:fbl_qp_relaxed_eq}
\end{align}
\end{subequations}
where $P \eqdef pI_2 \succ 0$ with $p>0$ being the scalar penalty on the slack variable. Moreover, the bias $\bm{c}(\mathbf x,r)$ is chosen as
\begin{equation}
\label{eq:new_bias}
\bm{c}(\mathbf x,r)
\;:=\;
\begin{bmatrix}
0 \\[0.3em]
-\epsilon_a\bigl(
\rho_{\parallel}(\mathbf x,r)\,s_{\parallel}(\mathbf x,r)
+\rho_{\perp}(\mathbf x,r)\,s_{\perp}(\mathbf x,r)
\bigr)
\end{bmatrix},
\end{equation}
where $\epsilon_a>0$ is a parameter to tune the effect of the bias on linear acceleration $u_2 = a$. It should be noted that the introduction of the slack variable $\bm{\delta}$ ensures that the QP constraint~\eqref{eq:fbl_qp_relaxed_eq} removes the singularity issue, and the QP thus becomes feasible even at zero velocity. 
However, a deadlock situation may still arise at zero velocity where the resulting control input $\mb{u}$ from the QP is zero for some reference inputs. 
To address this, we define the deadlock set
\begin{equation}
\label{eq:deadlock_set}
\mathcal Z \eqdef 
\Bigl\{(\mathbf x,r)\in\mathcal X\times\mathbb R^6:\ 
\begin{aligned}
&x_4=0,\ \bm\eta(\mathbf x,r)\neq \bm 0,\\
&\varphi(\mathbf x,r,p,\epsilon_a)=0
\end{aligned}
\Bigr\},
\end{equation}
where $\varphi(\mathbf x,r,p,\epsilon_a)$ is the output of a scalar map $\varphi:\mathcal X \times \mathbb R^6 \times \mathbb R \times \mathbb R \to \mathbb R$ such that
\begin{equation}\label{eq:deadlock_condition}
    \begin{aligned}
    \varphi(\mathbf x,r,p,\epsilon_a) :=  p\,s_\parallel(\mathbf x,r) + \epsilon_a&\bigl(
    \rho_{\parallel}(\mathbf x,r)\,s_{\parallel}(\mathbf x,r)\\
    &+\rho_{\perp}(\mathbf x,r)\,s_{\perp}(\mathbf x,r)\bigl).
\end{aligned}
\end{equation}
When $(\mathbf x,r)\in\mathcal Z$, the optimal input of~\eqref{eq:fbl_qp_relaxed} is zero. We now state the assumptions required to solve Problem~\ref{prob:tracking}.

\begin{assumption}\label{ass:bounded_v}
    Linear velocity is bounded,
    i.e., there exists $v_{\max}>0$ such that $|x_4|\le v_{\max}$ for all $\mathbf x\in\mathcal X$.
\end{assumption}

\begin{remark}
Assumption~\ref{ass:bounded_v} imposes a bound on state $x_4$, which is the translational velocity in the extended unicycle model. This is natural in practice, since mobile robots have actuator limits and rated maximum speeds.
\end{remark}

\begin{assumption}\label{ass:bounded_ref}
    The reference trajectory $\mathbf y_{\mathrm{ref}}\in C^2([0,\infty);\mathbb R^2)$
    has bounded acceleration, i.e.,
    there exists $Y_{\max}>0$ such that
    $\|\ddot{\mathbf y}_{\mathrm{ref}}(t)\|\le Y_{\max}$ for all $t\ge 0$.

\end{assumption}

\begin{assumption}\label{ass:start_deadlock_free}
    The state and reference signal pair is not in $\mc{Z}$, i.e., $(\mathbf x,r) \notin \mathcal Z$ for all time.
\end{assumption}

\begin{remark}
For any $\mathbf y_{\mathrm{ref}}\in C^2([0,\infty);\mathbb R^2)$ such that $x_4(t)=0$ and $\bm\eta(t)\neq \bm0,$ one can tune parameters $p>0,\ \epsilon_a>0$ so that $\varphi(x(t),r(t),p,\epsilon_a)\neq 0$.
Hence, the robot can always be initialized outside the deadlock set
$\mathcal Z$~\eqref{eq:deadlock_set}. 
\end{remark}
Next, we formally state the trajectory tracking problem that avoids the singularity at zero speed.  

\begin{problem}\label{prob:tracking}
     Given the extended unicycle model~\eqref{eq:extended_model} and a trajectory $\mathbf y_{\mathrm{ref}}$ with reference signal $r(t) \in \mathbb{R}^6$, design a controller $\mathbf{u} = \kappa(\mathbf{x}, r)$ under Assumptions~\ref{ass:bounded_v}--\ref{ass:start_deadlock_free}, such that
    \begin{enumerate}[label=\textbf{P1.\arabic*}, align=left]
        \item \label{prob:1.1} the map $\kappa$ is locally Lipschitz continuous, and
        \item \label{prob:1.2}  the tracking error state $\bm{\xi} \eqdef [\mathbf{e}^\top\ \dot{\mathbf{e}}^\top]^\top$ is uniformly ultimately bounded (UUB).
    \end{enumerate}
    \noindent
    We seek a feedback law in the sense that $\kappa$ depends on time only implicitly through the exogenous reference signal $r(t)$.
\end{problem}

\section{DFL-QP-based controller }
\label{sec:4}

In this section, we derive the feedback law induced by the relaxed QP~\eqref{eq:fbl_qp_relaxed} and show that it solves Problem~\ref{prob:tracking}. Define the augmented decision vector $\mathbf{w} \eqdef  \begin{bmatrix} \mathbf{u}^\top & \bm{\delta}^\top \end{bmatrix}^\top $. The quadratic cost term is defined by the block-diagonal matrix $H \in \mathbb{R}^{4 \times 4}$, given by $H \eqdef \operatorname{blkdiag}(Q, P)$. The linear cost term is defined by the augmented bias vector $\bm{\bar{c}}: \mathcal{X} \to \mathbb{R}^4$, where 
   $\bm{\bar{c}} \eqdef \begin{bmatrix} \bm{c}(\mathbf x,r)^\top & \bm{0}^\top \end{bmatrix}^\top$.
The equality constraint is defined by the matrix
$A \eqdef \begin{bmatrix} D(\mathbf{x}) & -I_2 \end{bmatrix}$.
With these definitions, the QP~\eqref{eq:fbl_qp_relaxed} can be equivalently written as
\begin{subequations}
\label{eq:standard_qp_form}
\begin{align}
\min_{\mathbf{w} \in \mathbb{R}^4} \quad & \frac{1}{2} \mathbf{w}^\top H \mathbf{w} + \bm{\bar{c}}^\top \mathbf{w}, \\
\text{s.t.} \quad & A\mathbf{w} = \boldsymbol{\eta}(\mathbf x,r).
\end{align}
\end{subequations}
Since $H\succ0$ and the constraint is affine, the solution is characterized by the KKT conditions~\cite{boyd2004convex}. Introducing the multiplier $\bm\lambda\in\mathbb R^2$, stationarity and feasibility give us the equations
$H\mathbf w+\bar{\bm c}+A^\top\bm\lambda=0$
and
$A\mathbf w=\bm\eta$, respectively. Solving these equations yields
\begin{equation}\label{eq:w_star}
\mathbf w^\star(\mathbf x,r)
=
-H^{-1}\bar{\bm c}
+
H^{-1}A^\top M(\mathbf x)^{-1}
\bigl(
\bm\eta+AH^{-1}\bar{\bm c}
\bigr),
\end{equation}
\begin{equation}
\label{eq:def_M}
\text{where }M(\mathbf x)\eqdef AH^{-1}A^\top
=
DQ^{-1}D^\top+P^{-1}.
\end{equation}
Since $D(\mathbf x)Q^{-1}D(\mathbf x)^\top\succeq0$ and $P^{-1}\succ0$, we have $M(\mathbf x)\succ0$ for all $\mathbf x\in\mathcal X$. Hence, $M(\mathbf x)^{-1}$ is well defined even when $D(\mathbf x)$ is singular, and the relaxed QP~\eqref{eq:fbl_qp_relaxed} remains feasible.  Moreover, since $D(\mathbf x)^{-1}$ contains terms that scale with $1/x_4$, exact DFL can require large control effort at low speeds. The relaxed QP mitigates this issue by allowing a bounded violation of the DFL constraint near singular configurations. We next state the first main result, which establishes local Lipschitz continuity of the proposed controller.
\begin{theorem}[]
\label{th:lip_con_closed_form}
For any fixed $r\in\mathbb R^6$, define
$\bar\kappa(\mathbf x,r)\eqdef \mathbf w^\star(\mathbf x,r)$.
Then the mapping $\mathbf x\mapsto \bar\kappa(\mathbf x,r)$
is locally Lipschitz continuous on $\mathcal X$.
\end{theorem}

\begin{proof}
For a fixed $r$, since $M(\mathbf x)\succ 0$ for all $\mathbf x$, the matrix inverse $\mathbf x\mapsto M(\mathbf x)^{-1}$ is $C^1$ on $\mathcal X$. Moreover, $\mathbf x\mapsto D(\mathbf x)$ is $C^\infty$ and $\mathbf x\mapsto \bm c(\mathbf x,r)$
is $C^1$ as it is a composition of trigonometric functions and smooth maps
defined in~\eqref{eq:scalar_proj} and~\eqref{eq:rhos}.
Therefore every term in \eqref{eq:w_star} is $C^1$ in $\mathbf x$. By \cite[Lemma 3.2]{khalil2002nonlinear}, every $C^1$ map is locally Lipschitz continuous.
\end{proof}

Now, let
$S_u\eqdef [I_2\;\;0_{2\times2}]$
and
$\mathbf u^\star(\mathbf x,r)\eqdef S_u\mathbf w^\star(\mathbf x,r)$. The resulting closed-loop system is
\begin{equation}
\label{eq:closed_loop_ext_sys}
\dot{\mathbf x}
=
F(\mathbf x)+G(\mathbf x)\mathbf u^\star(\mathbf x,r(t)).
\end{equation}
Since $r(\cdot)$ is continuous and $\mathbf u^\star$ is locally Lipschitz in $\mathbf x$, the closed-loop vector field is locally Lipschitz in $\mathbf x$ on compact time intervals. Thus,~\eqref{eq:closed_loop_ext_sys} admits a unique local solution~\cite[Thm. 3.1]{khalil2002nonlinear}, satisfying Problem~\ref{prob:1.1}.

Next, we study the tracking error. The key issue is at rest, where the decoupling matrix $D(\mathbf x)$ loses rank. The following lemma shows that, outside the deadlock set, the QP-based controller in~\eqref{eq:standard_qp_form} generates nonzero acceleration whenever a nonzero virtual input $\boldsymbol{\eta}$, defined in~\eqref{eq:virtual_input}, is required.

\begin{lemma}[]\label{lem:no_indef_stall} For any time $\tau\in\Real_{\geq 0}$, if the robot is at rest, i.e., $x_4(\tau)=0$ and $(\mathbf{x}(\tau), r(\tau)) \notin \mathcal{Z}$, the optimal acceleration $a^\star(\tau)$ is non-zero under Assumption~\ref{ass:start_deadlock_free}. Moreover, by continuity of the optimal acceleration
$a^\star$ along the closed-loop trajectory, there exist $t_{\mathrm{move}}>\tau$
and $v_{\min}>0$ such that $x_4(t)\neq 0$ for all
$t\in(\tau,t_{\mathrm{move}}]$ and
$|x_4(t_{\mathrm{move}})|\ge v_{\min}$.
\end{lemma}

\begin{proof}
    From~\eqref{eq:w_star},~\eqref{eq:scalar_proj}, and~\eqref{eq:rhos}, by plugging in $x_4=0$, we derive the optimal acceleration $u_2^\star=a^\star$ as
    \begin{equation}\label{eq:a_star_rest_rewrite}
a^\star=\frac{p\,s_{\parallel}+\epsilon_a\bigl(\rho_{\parallel}s_{\parallel}+\rho_{\perp}s_{\perp}\bigr)}{q_a+p}.
    \end{equation}
    Note that~\eqref{eq:a_star_rest_rewrite} is zero only if its numerator is zero, which is precisely $\varphi(\mathbf{x},r,p,\epsilon_a)$ in~\eqref{eq:deadlock_condition}. This holds for $\bm{\eta}(\bm{x}(\tau),r(\tau))= \bm{0}$, since $s_\parallel$ and $s_\perp$ are the projections of $\bm{\eta}$, but this implies we are tracking perfectly. For $\bm{\eta}(\bm{x}(\tau),r(\tau)) \neq 0$, we have $a^\star=0$ if and only if $(\mathbf{x}(\tau),r(\tau))\in\mathcal{Z}$. However, by Assumption~\ref{ass:start_deadlock_free}, the robot never enters the set $\mathcal{Z}$. Therefore, for any rest configuration such that $\bm{\eta}(\mathbf{x},r) \ne \bm{0}$, the robot will move. Moreover, since the reference trajectory is continuous (Assumption~\ref{ass:bounded_ref}), $a^\star$ is also continuous. It follows from the sign-preserving property of continuous functions that there is a finite time $t_{\mathrm{move}}$ in which the state $x_4=v$, governed by $\dot{x}_4 = a^\star$, will become non-zero, i.e., there exists a small scalar $v_{\min}>0$ such that $x_4(t)\neq 0$ for all $t\in(\tau, t_{\mathrm{move}}]$ and $|x_4(t_{\mathrm{move}})|\ge v_{\min}$.
\end{proof}

Next, we show that the tracking error under our proposed controller~\eqref{eq:standard_qp_form} is uniformly ultimately bounded.

\begin{theorem}
\label{thm:uub_after_exit}
Given an extended model for a unicycle robot~\eqref{eq:extended_model} with closed-loop dynamics~\eqref{eq:closed_loop_ext_sys} and a reference trajectory $\mathbf{y}_{\mathrm{ref}}$, the error state $\bm{\xi} \eqdef [\mathbf{e}^\top\ \dot{\mathbf{e}}^\top]^\top$ is uniformly ultimately bounded (UUB) under Assumptions~\ref{ass:bounded_v}--\ref{ass:start_deadlock_free}.
\end{theorem}

\begin{proof}
The proof shows that the optimal slack $\bm{\delta}^\star$ acts as a bounded disturbance on an exponentially stable error system. First, from~\eqref{eq:double_integrator},~\eqref{eq:fbl_qp_relaxed_eq}, and~\eqref{eq:virtual_input}, we have
$\ddot{\mathbf{y}} - \bm{\delta}^\star = \ddot{\mathbf{y}}_{\mathrm{ref}} - K_d \dot{\mathbf{e}} - K_p \mathbf{e}$. By rearranging terms, we get the error dynamics:
\[
\ddot{\mathbf{e}} + K_d\dot{\mathbf{e}} + K_p\mathbf{e} = \bm{\delta}^\star.
\]
With $\bm{\xi} = [\mathbf{e}^\top \ \dot{\mathbf{e}}^\top]^\top$, the error dynamics in state-space form are given by $\dot{\bm{\xi}} = A_\xi \bm{\xi} + B_\xi \bm{\delta}^\star$, where
\[
A_\xi = \begin{bmatrix} \bm{0} & I \\ -K_p & -K_d \end{bmatrix}, \quad B_\xi = \begin{bmatrix} \bm{0} \\ I \end{bmatrix}.
\]
Since $K_p\succ 0$ and $K_d\succ 0$, the matrix $A_\xi$ is Hurwitz, and thus the unforced error dynamics are exponentially stable. Next, we derive an upper bound on the slack term. From~\eqref{eq:w_star}, the slack variable can be written as
$\bm{\delta}^\star
= -R_\eta(\mathbf{x})\,\bm{\eta}(\mathbf{x},r)
  -R_c(\mathbf{x})\,\mathbf{c}(\mathbf{x},r),$ with $R_\eta(\mathbf{x})=P^{-1}M(\mathbf{x})^{-1}$, $
R_c(\mathbf{x})=P^{-1}M(\mathbf{x})^{-1}D(\mathbf{x})Q^{-1}
$. Here $M(\mathbf{x})$ is given by~\eqref{eq:def_M}.  
Now, on intervals where $|x_4|\ge v_{\min}>0$, as guaranteed by Lemma~\ref{lem:no_indef_stall}, the decoupling matrix in~\eqref{eq:decoupling_matrix} satisfies
\vspace{-0.05in}
\begin{equation}
\label{eq:DQDt_decomp}
D(\mathbf{x})Q^{-1}D(\mathbf{x})^\top
=
\frac{1}{q_a}\,\mathbf{n}_\parallel\mathbf{n}_\parallel^\top
+
\frac{x_4^2}{q_\omega}\,\mathbf{n}_\perp\mathbf{n}_\perp^\top.
\end{equation}
Since $\{\mathbf{n}_\parallel,\mathbf{n}_\perp\}$ is an orthonormal basis of $\mathbb{R}^2$,
\[
D(\mathbf x)Q^{-1}D(\mathbf x)^\top
\succeq
\varrho I_2,
\qquad
\varrho\eqdef
\min\left\{
\frac{1}{q_a},\frac{v_{\min}^2}{q_\omega}
\right\}.
\]
Since $P=pI_2$, we have $P^{-1}=\frac{1}{p}I_2$, from~\eqref{eq:def_M} we have
\[
M(\mathbf{x})
=
D(\mathbf{x})Q^{-1}D(\mathbf{x})^\top + \frac{1}{p}I_2
\succeq
\Big(\varrho+\frac{1}{p}\Big)I_2,
\]
which implies
$
\|M(\mathbf{x})^{-1}\|
\le
\frac{1}{\varrho+1/p}.
$
It follows that
\[
\|R_\eta(\mathbf{x})\|
=
\Big\|\frac{1}{p}M(\mathbf{x})^{-1}\Big\|
\le
\frac{1}{1+p\varrho}.
\]
Next, by sub-multiplicativity of the induced matrix norm,
\[
\|R_c(\mathbf{x})\|
\le
\|R_\eta(\mathbf{x})\|\,\|D(\mathbf{x})\|\,\|Q^{-1}\|
\le
\frac{\alpha}{1+p\varrho},
\]
where $\alpha \;\eqdef\; \sup_{\mathbf{x}\in\mathcal{X}} \|D(\mathbf{x})\|\,\|Q^{-1}\|$ is finite by Assumption~\ref{ass:bounded_v}.
Finally, since $|\rho_\parallel|\le 1$, $|\rho_\perp|\le 1$, and $|s_\parallel|,|s_\perp|\le \|\bm{\eta}\|$, the bias \eqref{eq:new_bias} satisfies
\[
\|\mathbf{c}(\mathbf{x},r)\|
=|c_2(\mathbf{x},r)|
\le
\epsilon_a\bigl(|s_\parallel|+|s_\perp|\bigr)
\le \epsilon_a\sqrt{2}\,\|\bm{\eta}(\mathbf{x},r)\|.
\]
By defining $k_\delta(p) \eqdef \frac{1+\alpha\epsilon_a\sqrt{2}}{1+p\varrho}$, we get
\begin{align*}
\|\bm{\delta}^\star\|
&\le
\|R_\eta\|\,\|\bm{\eta}\|
+
\|R_c\|\,\|\mathbf{c}\|
\;
\le
k_\delta(p) \
\|\bm{\eta}(\mathbf{x},r)\|,
\\
&\le k_\delta(p) \ \left( \|\ddot{\mathbf{y}}_{\mathrm{ref}}\| + \|K_p\mathbf{e} + K_d\dot{\mathbf{e}}\| \right).
\end{align*}
Applying the reference-acceleration bound in Assumption~\ref{ass:bounded_ref} and noting that
$\|K_p\mathbf{e}+K_d\dot{\mathbf{e}}\|
\le C_K\|\bm{\xi}\|,$
where $C_K$ denotes the induced norm of the constant matrix
$\begin{bmatrix}K_p & K_d\end{bmatrix}$, the slack term can be bounded as
\begin{equation}\label{eq:delta_bound}
    \|\bm{\delta}^\star\|\le k_\delta(p) \left( Y_{\max} + C_K \|\bm{\xi}\| \right).
\end{equation}
Now, consider the Lyapunov function candidate $V(\bm{\xi}) = \bm{\xi}^\top S \bm{\xi}$, where $A_\xi^\top S + S A_\xi = -I$ and $S \succ 0$ exists because $A_\xi$ is Hurwitz. The time derivative of $V$ is
\[
\dot{V} = -\|\bm{\xi}\|^2 + 2\bm{\xi}^\top S B_\xi \bm{\delta}^\star \le -\|\bm{\xi}\|^2 + 2\|S B_\xi\| \|\bm{\xi}\|\|\bm{\delta}^\star\|.
\]
Substituting the bound on $\|\bm{\delta}^\star\|$ from~\eqref{eq:delta_bound}, we get
$\dot{V} \le -\gamma\|\bm{\xi}\|^2+\beta\,\|\bm{\xi}\|,$ where $\gamma\eqdef 1-2\|SB_\xi\|\,C_K\,k_\delta(p)$ and $\beta\eqdef 2\|SB_\xi\|\,k_\delta(p)\,Y_{\max}$. On intervals where \eqref{eq:delta_bound} applies, we can choose $p$ such that $\gamma>0$. Moreover, when $\|\bm{\xi}\|\ge 2\beta/\gamma$, we have
$\dot{V} \le -\frac{\gamma}{2}\|\bm{\xi}\|^2.$
Therefore, $\dot V<0$ outside a compact set, which implies Uniform Ultimate Boundedness (UUB) of $\bm{\xi}(t)$~\cite[Thm.~4.18]{khalil2002nonlinear}. Hence, $\bm{\xi}(t)$ ultimately enters and remains in a compact neighborhood of the origin whose size is proportional to $\beta/\gamma$. This completes the proof.
\end{proof}

\begin{remark}
The size of the ultimate bound is reduced by increasing the slack penalty $p$. In particular, when the robot is away from zero velocity, the relaxed QP recovers the strict DFL constraint in the limit as $p\to\infty$. The parameter $\epsilon_a$ controls the strength of the acceleration bias used to escape rest configurations, while $l$ smooths the transition of the projection weights.
\end{remark} 

Note that the time $t_{\mathrm{move}}$ required to move from rest does not depend on $t_0 \ge 0$. Thus, the controller solves the tracking problem~\ref{prob:1.2} in the sense of Definition~\ref{def:UUB}. Moreover, Assumption~\ref{ass:start_deadlock_free} prevents the system from getting stuck in the deadlock set during stop-and-reverse maneuvers. Therefore, the relaxed DFL-QP controller solves Problem~\ref{prob:tracking} under Assumptions~\ref{ass:bounded_v}--\ref{ass:start_deadlock_free}.

\begin{remark}
While Assumption~\ref{ass:start_deadlock_free} may appear restrictive, in practice it only requires the system to be deadlock-free at rest \emph{initially}. This is mild for typical stop-and-reverse maneuvers with $C^2$ reference trajectories. Since $\ddot{\mathbf y}_{\mathrm{ref}}$ is continuous, the virtual input $\bm{\eta}(\mathbf{x},r)$ in~\eqref{eq:virtual_input} cannot jump instantaneously from a braking direction to a lateral direction at zero velocity. Thus, for smooth references, the deadlock condition is avoided for all time $t>0$.
\end{remark}

\section{Numerical Experiments}\label{sec:5}

We validate the proposed DFL-QP controller on the TurtleBot3 Waffle Pi in ROS2/Gazebo using its URDF model, which captures robot dynamics and actuator limits. The source code of our simulations is available online on \texttt{GitHub}\footnote{\url{https://github.com/gradslab/DFL_QP_Unicycle}}. All the experiments use a control update rate of $100$~Hz over a horizon of $T=20$~s. The applied signals are saturated according to the TurtleBot3 limits: $|x_4|\le0.26$~m/s, $|u_1|\le1.82$~rad/s, and $|u_2|\le1.0$~m/s$^2$. Although the analysis is developed for the unconstrained controller, actuator limits are imposed in simulation. A formal analysis of the constrained QP is left for future work.

\subsection{Half Figure-8}

The robot is initialized at $\mathbf{x}(0)=(-0.2,0.0,\pi,0.0)$ and compared against a traditional DFL controller with velocity reset~\cite{oriolo2002wmr}. The reference trajectory is
$\mathbf{p}_8(\tau) =
\begin{bmatrix}
\sin^2(\tau) &
\sin^2(\tau)\cos(\tau)
\end{bmatrix}^\top, \; 
\tau \in [0,2\pi],$ with cycle time $t_s$ parameterized as $
\tau(t) = \frac{2\pi}{t_s}(t \bmod t_s).$
This trajectory contains a cusp-like point at the origin, requiring the robot to stop and reverse direction. As shown in Figure~\ref{fig:trajectory_tracking_comp}, the traditional DFL controller~\cite{oriolo2002wmr} forces a U-turn at the cusp because it cannot pass through zero velocity, whereas the proposed DFL-QP controller tracks the reference through reversal. Figure~\ref{fig:trajectory_tracking_variables} further shows that the DFL-QP controller achieves lower tracking error by allowing $x_4$ to change sign.

\begin{figure}[!ht]
    \centering
    \includegraphics[width=\linewidth]{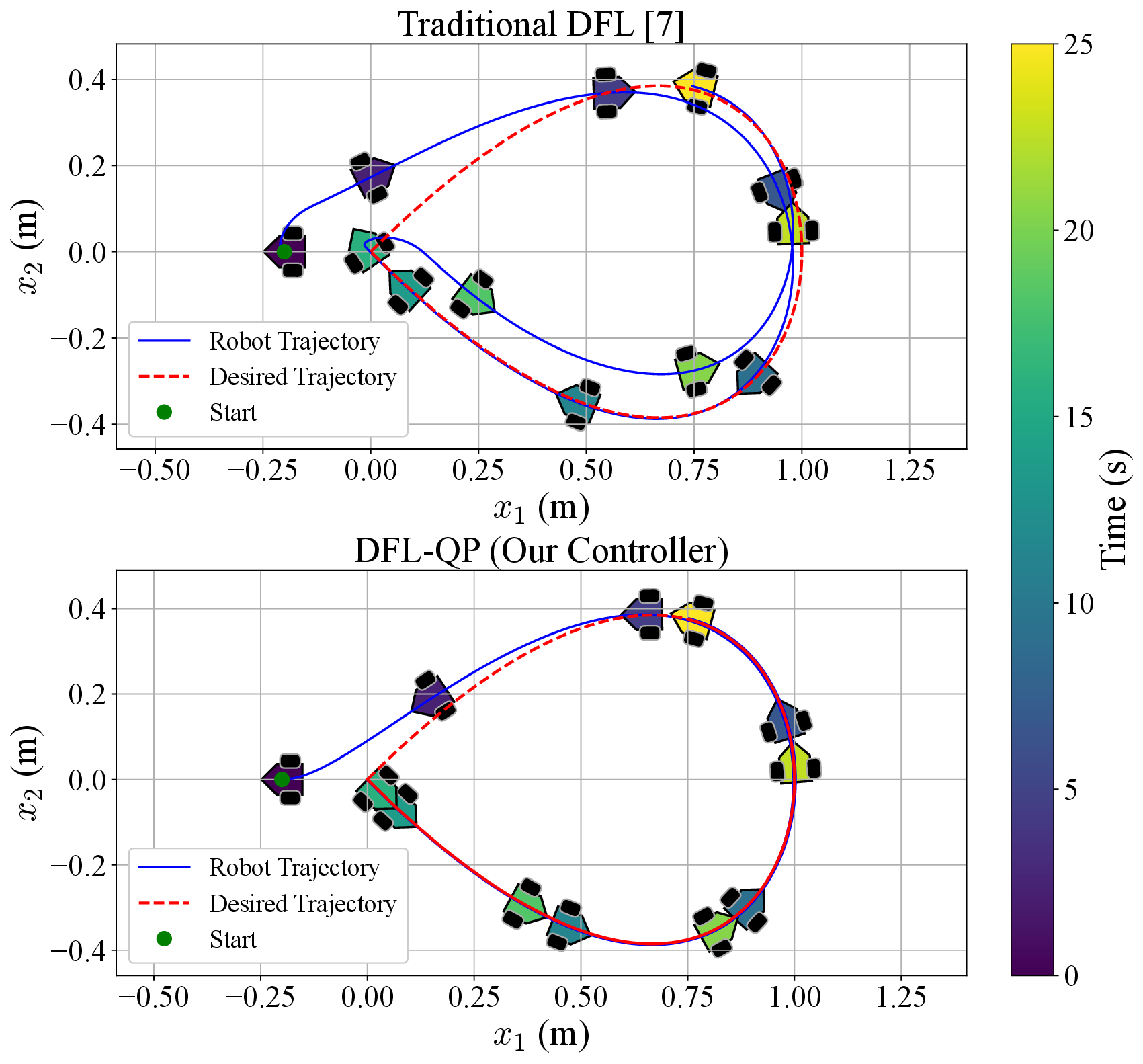}
    \vspace{-0.2in}
    \caption{Half figure-eight tracking comparison. The desired trajectory requires a stop-and-reverse maneuver at the cusp. The traditional DFL controller~\cite{oriolo2002wmr} loses accuracy near zero velocity, while the proposed DFL-QP tracks through the reversal. Poses are shown at $2$~s intervals.}
    \label{fig:trajectory_tracking_comp}
    \vspace{-0.1in}
\end{figure}

\begin{figure}[!ht]
    \centering
    \includegraphics[width=0.92\linewidth]{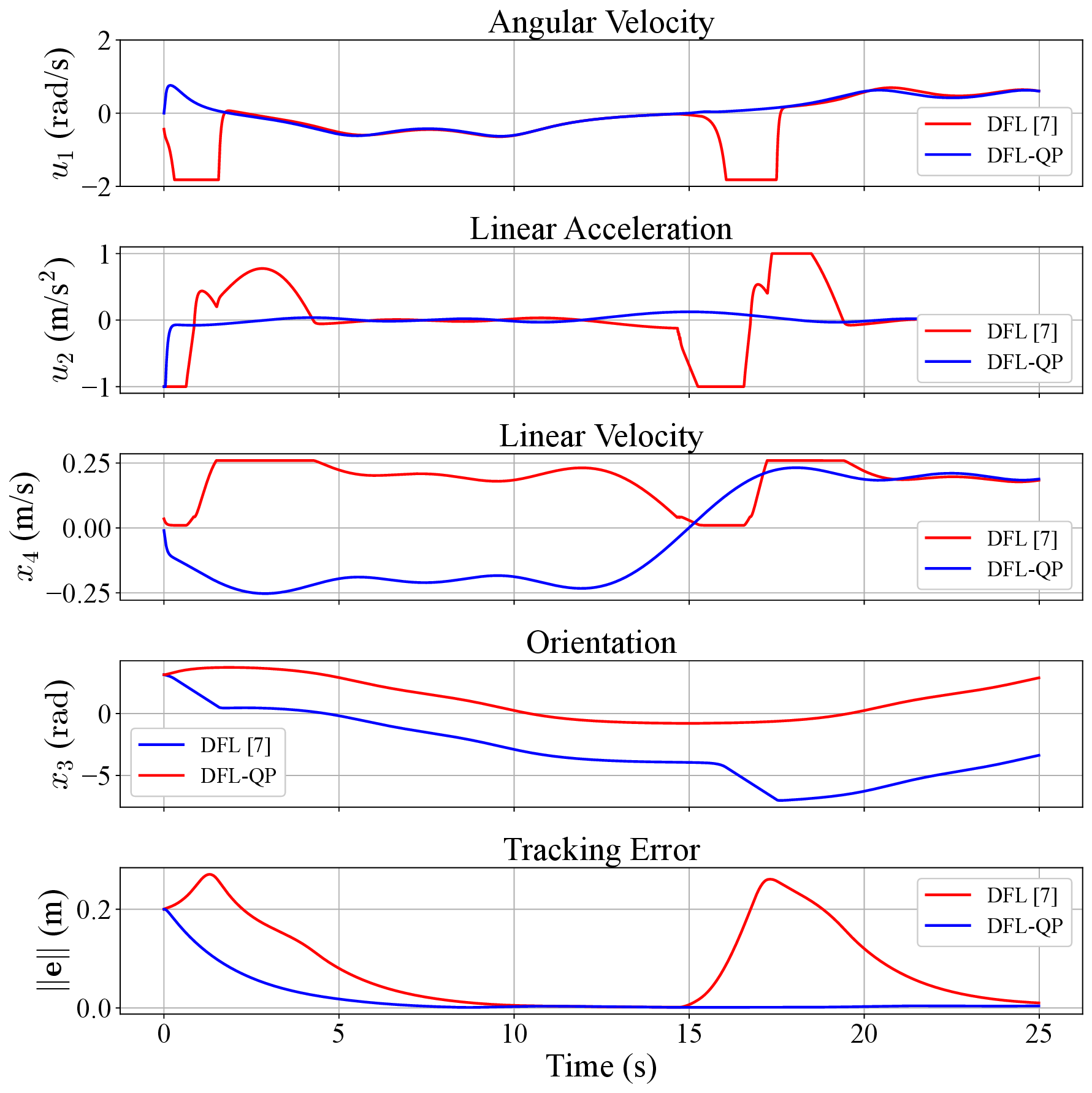}
    \vspace{-0.2in}
    \caption{State and control trajectories for the half figure-eight task. The proposed DFL-QP reverses linear velocity near $12.5$~s and achieves lower tracking error than the traditional DFL baseline.}
    \label{fig:trajectory_tracking_variables}
    \vspace{-0.1in}
\end{figure}

\subsection{Oscillating Line with Stop-and-Reverse}

We next evaluate repeated velocity reversals using an oscillating straight-line reference $\mathbf{y}_{\mathrm{ref}}: [0,T] \to \mathbb{R}^2$ defined as
$\mathbf{p}_{\ell}(\tau) =
\begin{bmatrix}
A \tau &
0
\end{bmatrix}^\top,
\;
\tau(t) = \sin\!\left(\frac{\pi t}{t_s}\right),
$
where $A > 0$ is the oscillation amplitude and $t_s$ denotes the half-period of oscillation. The corresponding derivatives are
$\dot{\tau}(t) = \frac{\pi}{t_s} \cos\!\left(\frac{\pi t}{t_s}\right), \;
\ddot{\tau}(t) = -\left(\frac{\pi}{t_s}\right)^2 \sin\!\left(\frac{\pi t}{t_s}\right).$
Thus, the reference is confined to the $x$-axis and requires periodic stops and reversals at $\tau=\pm1$. The robot is initialized at $\mathbf{x}(0)=(0.2,0.0,\pi,0.0)$, facing opposite the initial direction of motion. Figures~\ref{fig:trajectory_tracking_comp_line} and~\ref{fig:trajectory_tracking_variables_1} show that the traditional DFL controller loses effective angular control, leading to tracking failure. In contrast, the proposed DFL-QP controller crosses zero velocity smoothly and tracks the oscillating reference.

\begin{figure}[!ht]
    \centering
    \includegraphics[width=\linewidth]{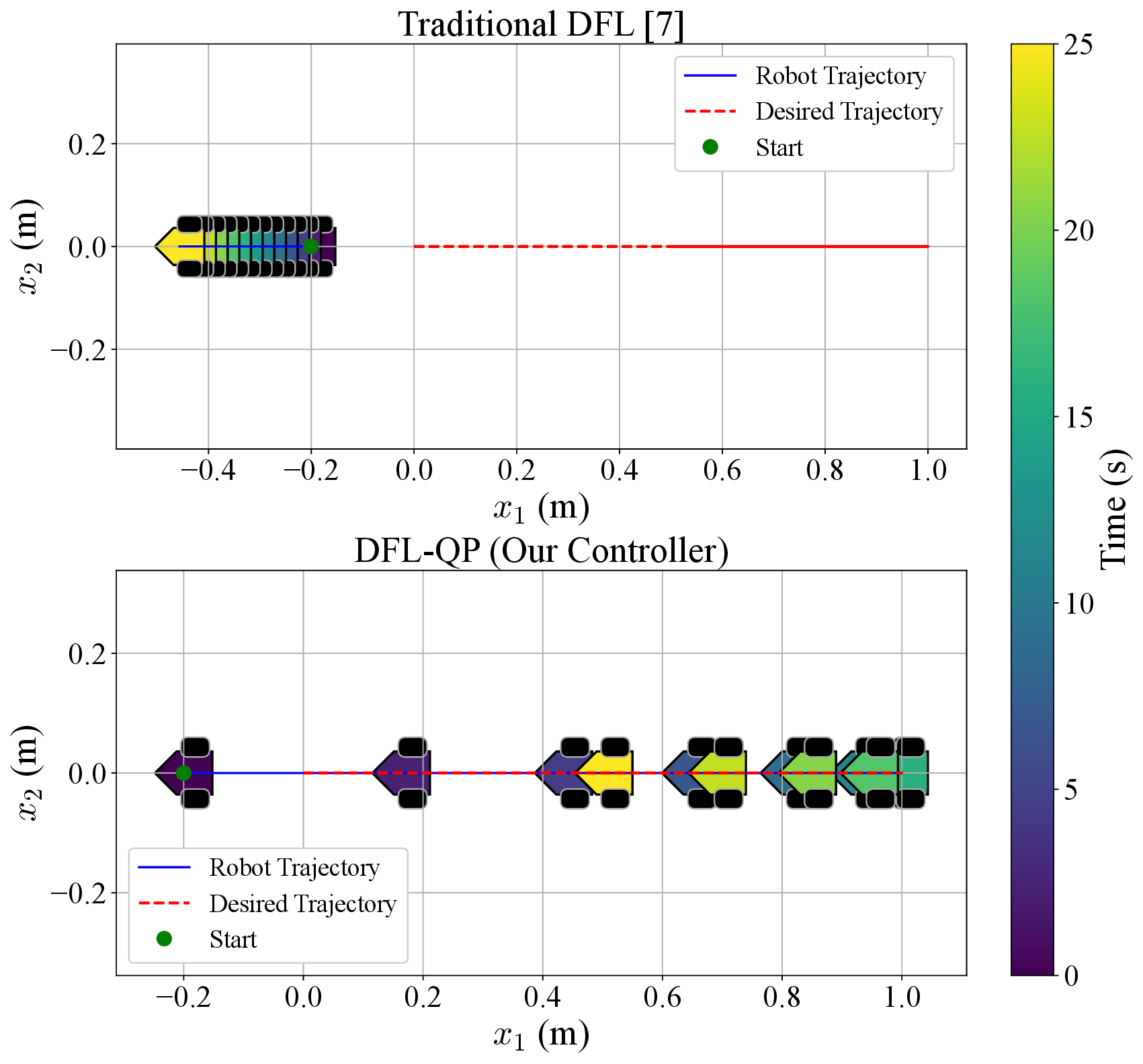}
    \vspace{-0.3in}
    \caption{Oscillating line tracking comparison. The desired motion requires repeated stop-and-reverse maneuvers. The traditional DFL controller~\cite{oriolo2002wmr} fails near zero velocity, while the proposed DFL-QP tracks through each reversal. Poses are shown at $2$~s intervals.}
    \label{fig:trajectory_tracking_comp_line}
    \vspace{-0.1in}
\end{figure}

\begin{figure}[!ht]
    \centering
    \includegraphics[width=0.92\linewidth]{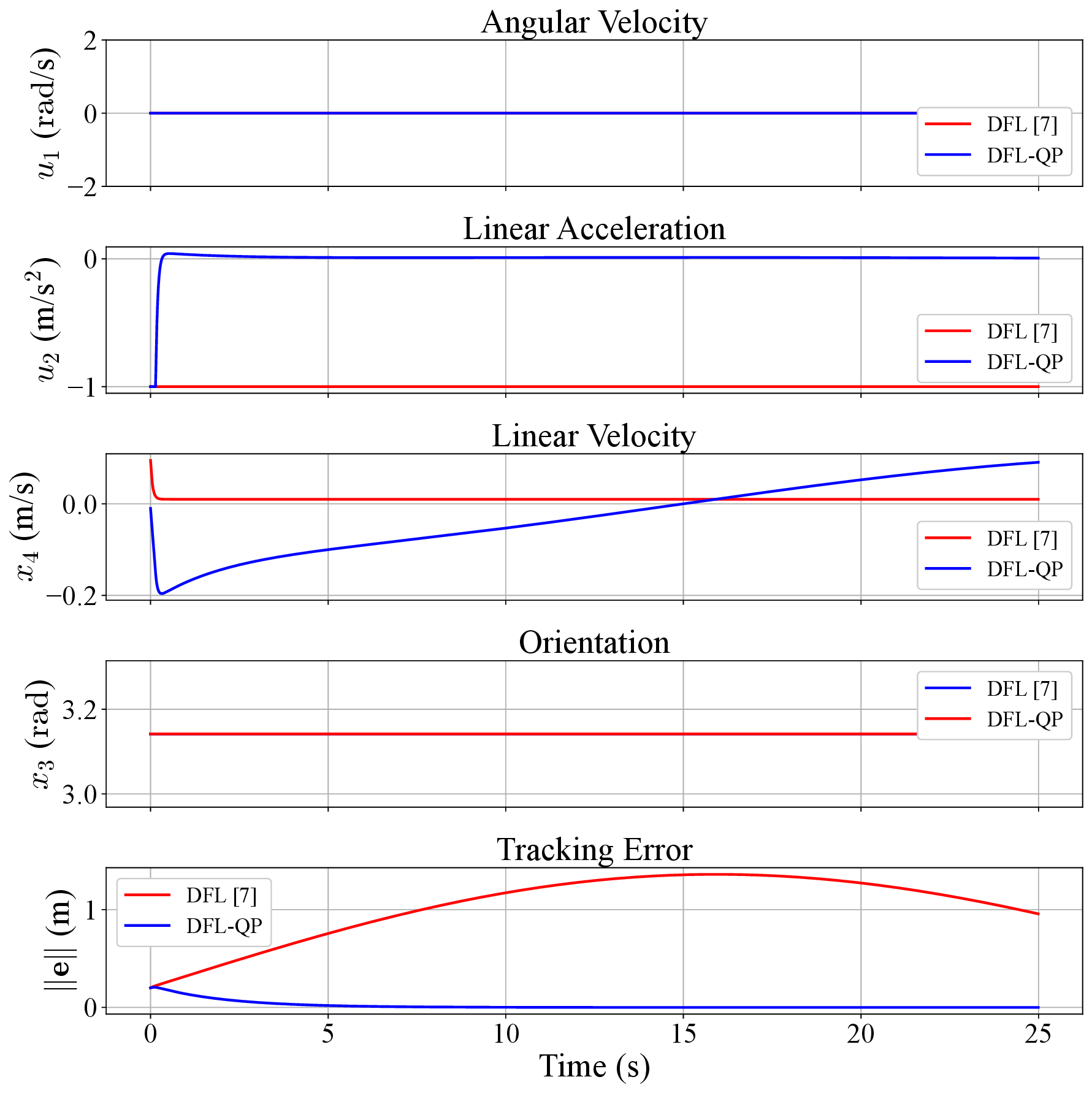}
    \vspace{-0.2in}
    \caption{State and control trajectories for the oscillating line task. The traditional DFL baseline~\cite{oriolo2002wmr} cannot smoothly cross zero velocity, while the proposed DFL-QP achieves smooth stop-and-reverse tracking.}
    \label{fig:trajectory_tracking_variables_1}
    \vspace{-0.15in}
\end{figure}


\section{CONCLUSION}\label{sec:6}


This paper introduces a Lipschitz-continuous, time-invariant controller for trajectory tracking of unicycle-modeled robots using a relaxed QP formulation to avoid singularities. We establish local Lipschitz continuity of the feedback controller and uniform ultimate boundedness of the tracking error. Numerical experiments in the ROS2-Gazebo physics engine on challenging trajectories that require a complete stop and direction reversal validate these results.

\bibliographystyle{IEEEtran}
\bibliography{IEEEabrv, myreferences}

\begin{thebibliography}{10}
\providecommand{\url}[1]{#1}
\csname url@rmstyle\endcsname
\providecommand{\newblock}{\relax}
\providecommand{\bibinfo}[2]{#2}
\providecommand\BIBentrySTDinterwordspacing{\spaceskip=0pt\relax}
\providecommand\BIBentryALTinterwordstretchfactor{4}
\providecommand\BIBentryALTinterwordspacing{\spaceskip=\fontdimen2\font plus
\BIBentryALTinterwordstretchfactor\fontdimen3\font minus \fontdimen4\font\relax}
\providecommand\BIBforeignlanguage[2]{{%
\expandafter\ifx\csname l@#1\endcsname\relax
\typeout{** WARNING: IEEEtran.bst: No hyphenation pattern has been}%
\typeout{** loaded for the language `#1'. Using the pattern for}%
\typeout{** the default language instead.}%
\else
\language=\csname l@#1\endcsname
\fi
#2}}

\bibitem{Padensurvey2016}
B.~Paden, M.~Čáp, S.~Z. Yong, D.~Yershov, and E.~Frazzoli, ``{A Survey of Motion Planning and Control Techniques for Self-Driving Urban Vehicles},'' \emph{IEEE Transactions on Intelligent Vehicles}, vol.~1, no.~1, pp. 33--55, 2016.

\bibitem{AkhWasNie2013Journal}
A.~Akhtar, C.~Nielsen, and S.~L. Waslander, ``Path following using dynamic transverse feedback linearization for car-like robots,'' \emph{IEEE Transactions on Robotics}, vol.~31, no.~2, pp. 269--279, April 2015.

\bibitem{hirschorn1987output}
R.~Hirschorn and J.~Davis, ``Output tracking for nonlinear systems with singular points,'' \emph{SIAM journal on control and optimization}, vol.~25, no.~3, pp. 547--557, 1987.

\bibitem{d1995control}
B.~d'Andr{\'e}a Novel, G.~Campion, and G.~Bastin, ``Control of nonholonomic wheeled mobile robots by state feedback linearization,'' \emph{The International journal of robotics research}, vol.~14, no.~6, pp. 543--559, 1995.

\bibitem{de1991exponential}
C.~C. De~Wit and O.~Sordalen, ``Exponential stabilization of mobile robots with nonholonomic constraints,'' in \emph{[1991] Proceedings of the 30th IEEE Conference on Decision and Control}.\hskip 1em plus 0.5em minus 0.4em\relax IEEE, 1991, pp. 692--697.

\bibitem{isidori1985nonlinear}
A.~Isidori, \emph{Nonlinear control systems: an introduction}.\hskip 1em plus 0.5em minus 0.4em\relax Springer, 1985.

\bibitem{oriolo2002wmr}
G.~Oriolo, A.~De~Luca, and M.~Vendittelli, ``{WMR control via dynamic feedback linearization: design, implementation, and experimental validation},'' \emph{IEEE Transactions on Control Systems Technology}, vol.~10, no.~6, pp. 835--852, 2002.

\bibitem{WanAkhSan2025}
\BIBentryALTinterwordspacing
N.~Wang, A.~Akhtar, and R.~G. Sanfelice, ``A safe hybrid control framework for car-like robot with guaranteed global path-invariance using a control barrier function,'' 2025. [Online]. Available: \url{https://arxiv.org/abs/2502.07136}
\BIBentrySTDinterwordspacing

\bibitem{samson2002control}
C.~Samson, ``Control of chained systems application to path following and time-varying point-stabilization of mobile robots,'' \emph{IEEE transactions on Automatic Control}, vol.~40, no.~1, pp. 64--77, 2002.

\bibitem{klanvcar2007tracking}
G.~Klan{\v{c}}ar and I.~{\v{S}}krjanc, ``Tracking-error model-based predictive control for mobile robots in real time,'' \emph{Robotics and autonomous systems}, vol.~55, no.~6, pp. 460--469, 2007.

\bibitem{astolfi1996discontinuous}
A.~Astolfi, ``Discontinuous control of nonholonomic systems,'' \emph{Systems \& control letters}, vol.~27, no.~1, pp. 37--45, 1996.

\bibitem{tomlin1998switching}
C.~Tomlin and S.~S. Sastry, ``Switching through singularities,'' \emph{Systems \& control letters}, vol.~35, no.~3, pp. 145--154, 1998.

\bibitem{ames2019control}
A.~D. Ames, S.~Coogan, M.~Egerstedt, G.~Notomista, K.~Sreenath, and P.~Tabuada, ``Control barrier functions: Theory and applications,'' in \emph{European control conference (ECC)}.\hskip 1em plus 0.5em minus 0.4em\relax IEEE, 2019, pp. 3420--3431.

\bibitem{bonnans2013perturbation}
J.~F. Bonnans and A.~Shapiro, \emph{Perturbation analysis of optimization problems}.\hskip 1em plus 0.5em minus 0.4em\relax Springer Science \& Business Media, 2013.

\bibitem{brockett1983asymptotic}
R.~W. Brockett \emph{et~al.}, ``Asymptotic stability and feedback stabilization,'' \emph{Differential geometric control theory}, vol.~27, no.~1, pp. 181--191, 1983.

\bibitem{wang2015simultaneous}
Y.~Wang, Z.~Miao, H.~Zhong, and Q.~Pan, ``{Simultaneous Stabilization and Tracking of Nonholonomic Mobile Robots: A Lyapunov-Based Approach},'' \emph{IEEE Transactions on Control Systems Technology}, vol.~23, no.~4, pp. 1440--1450, 2015.

\bibitem{khalil2002nonlinear}
H.~K. Khalil, \emph{Nonlinear systems}.\hskip 1em plus 0.5em minus 0.4em\relax Prentice hall Upper Saddle River, NJ, 2002, vol.~3.

\bibitem{boyd2004convex}
S.~P. Boyd and L.~Vandenberghe, \emph{Convex optimization}.\hskip 1em plus 0.5em minus 0.4em\relax Cambridge university press, 2004.

\end{thebibliography}

\end{document}